\documentclass[conf]{ieeeconf}%

\IEEEoverridecommandlockouts                              %

\usepackage{xr}

\newcommand{\omitted}[1]{}%

\title{%
Distributed Online Submodular Maximization under Communication Delays: A Simultaneous Decision-Making Approach
}
\author{Zirui Xu, Vasileios Tzoumas$^\dagger$
	\thanks{%
 $^\dagger$Department of Aerospace Engineering, University of Michigan, Ann Arbor, MI 48109 USA;  {\tt\footnotesize \{ziruixu,vtzoumas\}@umich.edu}} 
    \thanks{This work was supported by NSF CAREER Award No. 2337412 and ARO Early Career Program Award W911NF-25-1-0280.}
}

\usepackage{cite}

\usepackage{comment}
\usepackage{siunitx}
\usepackage{relsize}
\usepackage{ifthen}
\usepackage[colorinlistoftodos]{todonotes}

\usepackage[vlined,ruled,linesnumbered]{algorithm2e}
\usepackage{graphics} %
\usepackage{rotating}
\usepackage{color}
\usepackage{enumerate}
\usepackage[T1]{fontenc}
\usepackage{psfrag}
\usepackage{epsfig} %
\usepackage{booktabs}
\usepackage{graphicx,url}
\usepackage{multirow}
\usepackage{array}
\usepackage{latexsym}
\usepackage{amsfonts}
\usepackage{amsmath}
\usepackage{amssymb}
\usepackage{xstring}
\usepackage{algorithmic}
\usepackage{multirow}
\usepackage{xcolor}
\usepackage{prettyref}
\usepackage{flexisym}
\usepackage{bigdelim}
\usepackage{breqn} %
\usepackage{listings}
\let\labelindent\relax
\usepackage{xspace}
\usepackage{bm}
\graphicspath{{./figures/}}
\usepackage{tikz}
\usetikzlibrary{matrix,calc}
\usepackage{lipsum}
\usepackage{mdwlist}

\let\itemize\undefined
\makecompactlist{itemize}{stditemize}
\usepackage{enumitem}
\usepackage{caption}

\usepackage{amsthm}
\usepackage{mathtools}
\usepackage{mleftright}

\makeatletter
\let\NAT@parse\undefined
\makeatother
\usepackage{hyperref}
\usepackage{cleveref}
\allowdisplaybreaks

\hypersetup{%
  colorlinks=true,
  linkcolor=blue,
  filecolor=magenta,      
  urlcolor=black,
  citecolor=red,
  linkbordercolor={0 0 1}
}

\newtheorem{theorem}{Theorem}
\newtheorem{problem}{Problem}

\newtheorem{definition}{Definition}
\newtheorem{proposition}{Proposition}

\newtheorem{remark}{Remark}

\newcommand{\bdmath}{\begin{dmath}}
\newcommand{\edmath}{\end{dmath}}
\newcommand{\beq}{\begin{equation}}
\newcommand{\eeq}{\end{equation}}
\newcommand{\bdm}{\begin{displaymath}}
\newcommand{\edm}{\end{displaymath}}
\newcommand{\bea}{\begin{eqnarray}}
\newcommand{\eea}{\end{eqnarray}}
\newcommand{\beal}{\beq \begin{array}{lll}}
\newcommand{\eeal}{\end{array} \eeq}
\newcommand{\beas}{\begin{eqnarray*}}
\newcommand{\eeas}{\end{eqnarray*}}
\newcommand{\ba}{\begin{array}}
\newcommand{\ea}{\end{array}}
\newcommand{\bit}{\begin{itemize}}
\newcommand{\eit}{\end{itemize}}
\newcommand{\ben}{\begin{enumerate}}
\newcommand{\een}{\end{enumerate}}

\newcommand{\calA}{{\cal A}}
\newcommand{\calB}{{\cal B}}
\newcommand{\calC}{{\cal C}}

\newcommand{\calE}{{\cal E}}

\newcommand{\calG}{{\cal G}}

\newcommand{\calN}{{\cal N}}

\newcommand{\calV}{{\cal V}}

\definecolor{myblue}{RGB}{65 105 225}

\newcommand{\hide}[1]{}

\newcommand{\hiddenText}{{\color{gray} hidden text.}}
\newcommand{\hideWithText}[1]{\hiddenText}

\newcommand{\opt}{^{\star}}

\DeclareRobustCommand{\scenario}[1]{%
  \ifmmode
    \mathsf{#1}%
  \else
    \texorpdfstring{{\fontsize{8.9}{9}\selectfont\sffamily #1}\xspace}{#1}%
  \fi
}

\newcommand{\ie}{\emph{i.e.},\xspace}
\newcommand{\eg}{\emph{e.g.},\xspace}
\newcommand{\myin}{\, \in \,}

\newcommand{\sg}{\scenario{SG}}
\newcommand{\banalg}{\scenario{BSG}}

\newcommand{\myParagraph}[1]{{\bf #1.}\xspace}

\renewcommand{\opt}{\scenario{OPT}}

\newcommand{\curv}{\kappa}
\newcommand{\ourcurv}{\scenario{coin}}

\newcommand{\alg}{\scenario{DOG}}
\newcommand{\actionsel}{\scenario{ActSel}}

\newcommand{\elem}{v}

\newcommand{\distfsf}{p}

\newcommand{\solopt}{\calA^{\opt}}

\newcommand{\Reg}{\operatorname{Reg}_{T}}

\PassOptionsToPackage{end}{algorithmic}

\begin{document}

\maketitle

\thispagestyle{empty}
\pagestyle{empty}

\begin{abstract}
We provide a distributed online algorithm for multi-agent submodular maximization under communication delays. We are motivated by the future distributed information-gathering tasks in unknown and dynamic environments, where utility functions naturally exhibit the diminishing-returns property, \ie submodularity. Existing approaches for online submodular maximization either rely on sequential multi-hop communication, resulting in prohibitive delays and restrictive connectivity assumptions, or restrict each agent's coordination to its one-hop neighborhood only, thereby limiting the coordination performance. 
To address the issue, we provide the Distributed Online Greedy ({\fontsize{8}{8}\selectfont\sffamily DOG}) algorithm, which integrates tools from adversarial bandit learning with delayed feedback to enable simultaneous decision-making across arbitrary network topologies. We provide the approximation performance of {\fontsize{8}{8}\selectfont\sffamily DOG} against an optimal solution, capturing the suboptimality cost due to decentralization as a function of the network structure. Our analyses further reveal a trade-off between coordination performance and convergence time, determined by the magnitude of communication delays. By this trade-off, {\fontsize{8}{8}\selectfont\sffamily DOG} spans the spectrum between the state-of-the-art fully centralized online coordination approach~\cite{xu2023bandit} and fully decentralized one-hop coordination approach~\cite{xu2026self}.

\end{abstract}

\section{Introduction}\label{sec:Intro}
Multi-agent systems of the future will increasingly rely on agent-to-agent communication to coordinate tasks such as target tracking~\cite{xu2023bandit}, environmental mapping~\cite{atanasov2015decentralized}, and area monitoring~\cite{corah2018distributed}. These tasks are often modeled as maximization problems of the form
\begin{equation}\label{eq:intro}
	\max_{a_{i,t}\,\in\,\mathcal{V}_i,\,  \forall\, i\,\in\, \calN}\
	f_t(\,\{a_{i,t}\}_{i\myin \calN}\,), \;\;\;t=1,2,\dots,
\end{equation}
across the robotics, control, and machine learning communities, where $\calN$ denotes the set of agents, $a_{i,t}$ denotes agent $i$'s chosen action at time $t$, $\calV_i$ denotes agent $i$'s set of available actions, and $f_t\colon 2^{\prod_{i \in \calN}\calV_i}\mapsto\mathbb{R}$ denotes the objective function that captures the task utility (global objective)~\cite{krause2008near,singh2009efficient,tokekar2014multi,atanasov2015decentralized,gharesifard2017distributed,marden2017role,grimsman2019impact,corah2018distributed,schlotfeldt2021resilient,du2022jacobi,rezazadeh2023distributed,robey2021optimal,xu2025communication}. In resource allocation and information gathering applications, $f_t$ is \textit{submodular}~\cite{fisher1978analysis},  a diminishing-returns property~\cite{krause2008near}. For example, in target monitoring with multiple reorientable cameras, $\calN$ is the set of cameras, $\calV_i$ represents the possible orientations of each camera, and $f_t$ measures the number of distinct targets observed within the joint field of view.

The optimization problem in~\cref{eq:intro} is NP-hard~\cite{Feige:1998:TLN:285055.285059}, but polynomial-time algorithms with provable approximation guarantees exist when the $f_t$ is submodular. A classical example is the \emph{Sequential Greedy} (\sg) algorithm~\cite{fisher1978analysis}, which guarantees a $1/2$-approximation ratio. Many multi-agent tasks, including target tracking, collaborative mapping, and monitoring, can be cast as submodular coordination problems. Consequently, \sg and its variants have been widely adopted in the controls, machine learning, and robotics literature~\cite{krause2008near,singh2009efficient,tokekar2014multi,atanasov2015decentralized,gharesifard2017distributed,grimsman2019impact,corah2018distributed,schlotfeldt2021resilient,liu2021distributed,robey2021optimal,rezazadeh2023distributed,konda2022execution,krause2012submodular,xu2025communication}.

In this paper, we focus on applications where the environment is unpredictable and partially observable, and where the agents must solve the optimization problem in~\cref{eq:intro} via agent-to-agent communication, \eg via mesh networks. Such optimization settings are challenging since, respectively: (i) $f_t(\cdot)$ is unknown a priori, necessitating online optimization approaches where the agents jointly plan actions using only retrospective feedback (bandit feedback) {\cite{xu2023bandit}}; and (ii) the state-of-the-art agent-to-agent communication speeds are slow compared to wired connections or 5G~\cite{sadler2024low}, necessitating novel decentralized optimization paradigms that rigorously sacrifice near-optimality for scalability~\cite{xu2025communication}.  In such challenging optimization settings, \sg and its variants offer no performance guarantees~\cite{xu2023bandit}.  

In more detail, the related work on submodular maximization in unpredictable and partially observable environments, and in the decentralized optimization context of limited communication speeds, is as follows:

\paragraph{Online submodular optimization} In unpredictable settings, such as target tracking with maneuvering targets whose intentions are unknown~\cite{sun2020gaussian}, drones cannot forecast the future to evaluate $f_t$ in advance. Instead, they must coordinate actions online using retrospective feedback. The challenge deepens under partial observability: with limited sensing (\eg drones tracking targets within a restricted field of view), drones can evaluate only the reward of executed actions but not the alternative rewards of unselected ones. This bandit feedback~\cite{lattimore2020bandit} prevents agents from fully exploiting past information, thus hindering the design of near-optimal coordination strategies in such environments. To this end, sequential coordination algorithms leveraging online feedback have been proposed in~\cite{xu2023online,xu2023bandit}, which provide guaranteed suboptimality against robots' optimal time-varying actions in hindsight.  These algorithms extend \sg to the bandit setting, expanding tools from the literature on \textit{tracking the best expert} (\eg the \scenario{EXP3-SIX} algorithm~\cite{neu2015explore}) to the multi-agent setting upon accounting for the submodular structure of the optimization problem.

\paragraph{Online submodular optimization under low communication speeds} But online sequential algorithms~\cite{xu2023online,xu2023bandit}, similar to their offline sequential greedy counterparts~\cite{krause2008near,singh2009efficient,tokekar2014multi,atanasov2015decentralized,gharesifard2017distributed,grimsman2019impact,corah2018distributed,schlotfeldt2021resilient,liu2021distributed,robey2021optimal,rezazadeh2023distributed,konda2022execution,krause2012submodular}, cannot scale under real-world communication conditions~\cite{xu2025communication}. Since they perform sequential multi-hop communication over (strongly) connected networks to enable near-optimality, they can cause excessive communication delays between consecutive time steps $t$ and $t+1$~\cite{xu2025communication}.
In particular, for these state-of-the-art algorithms (offline and online variants), the communication delays increase quadratically or even cubically as the number of agents increases~\cite{xu2025communication}. 
For example, the Bandit Sequential Greedy (\banalg) algorithm in~\cite{xu2023bandit} requires: (i) a communication complexity that is cubic in the number of agents at each time step (decision round) over a worst-case directed network, %
such that all agents can obtain the feedback of their selected actions, and (ii) the number of decision rounds being quadratic in the number of agents such that the algorithm can converge. That is, it takes up to a \textit{quintic} time in the number of agents for \banalg to achieve near-optimal coordination performances~\cite[Theorem~6]{xu2026self}.

To address the communication issues above, novel distributed optimization algorithms have been proposed that achieve linear time complexity in the number of agents.  To this end, for example, they (i) require each agent to coordinate with \textit{one-hop neighbors only}, and (ii) operate over {arbitrary network topologies}. Such an algorithm is the Resource-Aware distributed Greedy (\scenario{RAG}) algorithm~\cite{xu2025communication}.
While \scenario{RAG} matches the performance of Sequential Greedy in fully centralized networks, for arbitrary network topologies, it suffers a suboptimality cost, as a function of the network topology.
However, \scenario{RAG} applies to offline optimization only, where $f_t(\cdot)$ is know a priori, instead of online.  Another example is the \actionsel subroutine proposed in~\cite{xu2026self}, which extends \scenario{RAG} to the online settings

While the online approach \actionsel enables simultaneous decisions by avoiding sequential communication and yet still provides an asymptotically the same suboptimality guarantee as \scenario{RAG}, it has even larger potentials: the reliance on strictly local information limits the coordination performance, as agents cannot access broader information from beyond their immediate neighbors. Therefore, the following research question arises: 
\textit{Under communication delays, how does each agent coordinate with others beyond its immediate neighborhood to maximize action coordination performance without sacrificing decision speed?}

\myParagraph{Contributions}
In this paper, we develop an online optimization algorithm that allows agents to exploit multi-hop communication in order to leverage information from beyond immediate neighbors, such that the performance gap between distributed and centralized coordination can be minimized. To address the inevitable multi-hop communication delays, we leverage techniques from \textit{bandit learning with delayed feedback}~\cite{thune2019nonstochastic}, enabling agents to select asymptotically near-optimal actions despite outdated feedback.
The algorithm has the following properties:

\setcounter{paragraph}{0}
	
\paragraph{Approximation Performance} 
The algorithm provides a suboptimality bound against an optimal solution to \cref{eq:intro}.  Particularly, the bound captures the suboptimality cost due to decentralization, as a function of each agent's multi-hop coordination neighborhood. For example, as long as the network is connected (not necessarily fully connected), then the algorithm's approximation bound is $1/2$ because every agent can receive information from all others via multi-hop communication. In particular, \alg's bound is lower than \banalg's $1/(1+\kappa_f)$ due to decentralization, and better than \actionsel's $1/(1+\kappa_f) - \sum_{i\in\calN}\scenario{coin}(\calN_i^{(1)})$ since multi-hop communication enlarges the coordination neighborhood, where $\calN_i^{(1)}$ denotes $i$'s one-hop neighborhood 
(\Cref{sec:bound}).  

\paragraph{Convergence Time} The algorithm enables agents to \textit{simultaneously select actions} for every $(2\tau_f+ \tau_c\, \bar{d})$ time, where $\tau_f$ and $\tau_c$ are the times for one function evaluation and transmitting one action for one-hop communication, respectively among all $i\in\calN$. The convergence time of \alg is $\Tilde{O}\left[(\tau_f+\tau_c\,\bar{d})\,|\calN|^2\,\max_{i\in\calN}\, (|\calV_i|\, +\, d_i)\,/\,\epsilon\right]$, slower than the \actionsel subroutine in~\cite{xu2026self} by a factor of $\bar{d}^2$, while up to $|\calN|^3$ faster than \banalg~\cite{xu2023bandit} (\Cref{sec:resource-guarantees}).

\section{Distributed Online Submodular Maximization Under Communication Delays}\label{sec:problem}

We present the problem formulation, using the notation:
\begin{itemize}[leftmargin=*]
    \item $\calV_\calN \triangleq \prod_{i\myin \calN} \,\calV_i$ is the cross product of sets $\{\calV_i\}_{i\myin \calN}$.
    \item $[T]\triangleq\{1,\dots,T\}$ for any positive integer $T$;
    \item $f(\,a\,|\,\calA\,)\triangleq f(\,\calA \cup \{a\}\,)-f(\,\calA\,)$ is the marginal gain of set function $f:2^\calV\mapsto \mathbb{R}$ for adding $a \in \calV$ to $\calA \subseteq\calV$.
    \item $|\calA|$ is the cardinality of a discrete set $\calA$. 
\end{itemize}
We also use the following framework about the agents' communication network and their global objective $f$.

\myParagraph{Communication network} 
The distributed communication network $\calG=\{\calN, \calE\}$ \textit{can be directed and even disconnected}, where $\calE$ is the set of communication channels. When $\calG$ is fully connected (all agents receive information from all others), we call it \textit{fully centralized}. In contrast, when $\calG$ is fully disconnected (all agents are isolated, receiving information from no other agent), we call it \textit{fully decentralized}.

\myParagraph{Communication neighborhood}  
When a communication channel exists from agent $j$ to $i$, \ie $(j\rightarrow i) \in \calE$, $i$ can receive, store, and process information from $j$. The set of all agents from which $i$ can receive information, possibly through multi-hop communication, is denoted by $\calN_i$. Thus, $\calN_i$ represents agent $i$'s \textit{$\infty$-hop in-neighborhood}. For simplicity, we refer to $\calN_i$ as agent $i$'s \textit{neighborhood} and assume it remains constant over $[T]$. Information originating from different neighbors $j\in\calN_i$ may take varying amounts of time to reach $i$, depending on the message size, communication data rate, and the number of hops from $j$ to $i$.

\myParagraph{Communication delay}
The communication delay is determined by the radius of agent $i$'s neighborhood, \ie the number of edges from the furthest multi-hop neighbor to $i$. We denote this delay by $d_i$. 
    In particular, $d_i$ is also the delay for agent $i$ to receive the reward of selecting action $a_{i,t}$ at $t$.
    That is, the value of $r_{a_{i,t},\,t}$ will be available at $t+d_i$.

\begin{definition}[Normalized and Non-Decreasing Submodular Set Function{~\cite{fisher1978analysis}}]\label{def:submodular}
A set function $f:2^\calV\mapsto \mathbb{R}$ is \emph{normalized and non-decreasing submodular} if and only if 
\begin{itemize}[leftmargin=*]
\item (Normalization) $f(\,\emptyset\,)=0$;
\item (Monotonicity) $f(\,\calA\,)\leq f(\,\calB\,)$, $\forall\,\calA\subseteq \calB\subseteq \calV$;
\item (Submodularity) $f(\,s\,|\,\calA\,)\geq f(\,s\,|\,{\mathcal{B}}\,)$, $\forall\,\calA\subseteq {\mathcal{B}}\subseteq\calV$ and $s\in \calV$.
\end{itemize}
\end{definition}

Intuitively, if $f(\calA)$ captures the number of targets tracked by a set $\calA$ of sensors, then the more sensors are deployed, more or the same targets are covered; this is the non-decreasing property.  Also, the marginal gain of tracked targets caused by deploying a sensor $s$ \emph{drops} when \emph{more} sensors are already deployed; this is the submodularity~property.

\begin{definition}[2nd-order Submodular Set Function{~\cite{crama1989characterization,foldes2005submodularity}}]\label{def:conditioning}
$f:2^\calV\mapsto \mathbb{R}$ is \emph{2nd-order submodular} if and only if 
\begin{equation}\label{eq:conditioning}
    f(s\,|\,\calC) - f(s\,|\,\calA\cup\calC) \geq f(s\,|\,\calB\cup\calC) - f(s\,|\,\calA\cup\calB\cup\calC),
\end{equation}
for any \emph{disjoint} $\calA, \calB, \calC\subseteq \calV$ ($\calA \cap \calB \cap \calC =\emptyset$) and  $s\in\calV$.
\end{definition}

Intuitively, if $f(\,\calA\,)$ captures the number of targets tracked by a set $\calA$ of sensors, then \emph{marginal gain of the marginal gains} drops when more sensors are already deployed.

\begin{problem}[Distributed Online Submodular Maximization under Communication Delays]
\label{pr:online}
At each time step $t\in[T]$, given the multi-hop neighborhood $\calN_i$, each agent $i\in\calN$ needs to select an action $a_{i,t}$ to jointly solve 
\begin{equation}\label{eq:problem}
    \max_{a_{i,t}\in\mathcal{V}_i, \forall\, i\in \calN} \sum_{t=1}^T f_t(\,\{a_{i,t}\}_{i\myin \calN}\,)
\end{equation}
where $f_t\colon 2^{\calV_{\calN}}\mapsto \mathbb{R}$ is a normalized, non-decreasing submodular, and 2nd-order submodular set function, and
each agent $i$ can access the value of $f_t(\,\calA\,)$ only after it has selected $a_{i,t}$ at time $t$ and received $\{a_{j,t}\}_{j\in \calN_{i}}$ at time $t+d_i$, $\forall\,\calA\subseteq \{a_{i,t}\}\cup\{a_{j,t}\}_{j\in \calN_{i}}$. 
\end{problem}

\Cref{pr:online} is a generalization to the problem in~\cite{xu2023bandit} by considering (i) the impact of communication delays, and (ii) an arbitrary rather than a connected communication network. Moreover, \Cref{pr:online} differs from~\cite{xu2025communication} by (i) addressing unknown environments, and (ii) allowing for multi-hop instead of merely one-hop communications. 

The action coordination performance in \Cref{pr:online} highly depends on the network $\{\calN_i\}_{i\in\calN}$: it will improve as the network becomes more centralized (from all agents coordinating with none to all agents coordinating with all). 
For example, consider the target monitoring scenario with multiple reorientable cameras: as the cameras become more centralized, they can each coordinate with more others to avoid covering the same targets, thus improving the total number of covered targets. Therefore, in this paper, we propose to adopt multi-hop communication to maximize each agent's information access over the distributed communication network. To mitigate the influence of communication delays to the action coordination frequency, we leverage tools from bandit learning with delayed feedback, which will be shown in the next section. %

\section{Distributed Online Greedy Algorithm (\alg)} \label{sec:algorithm}

\setlength{\textfloatsep}{3mm}
\begin{algorithm}[t]
	\caption{Distributed Online Greedy (\alg) for Agent $i$
	}
	\begin{algorithmic}[1]
		\REQUIRE \!Number of time steps $T$, agent $i$'s action set $\calV_i$, agent $i$'s in-neighborhood $\calN_i$, {communication delay $d_i$}. %
		\ENSURE \!Agent $i$'s action $a_{i,\, t}$, $\forall t\in[T]$.
		\medskip
            \STATE $\eta_i\gets\sqrt{\log{|\calV_i|}\,/\,[(|\calV_i|+d_i)\,T]}$;
            \STATE $w_{1}\gets\left[w_{1,1}, \dots, w_{|\calV_i|,1}\right]^\top$ with $w_{v,1}=1, \forall a\in \calV_i$;
            \FOR {\text{each time step} $t\in [T]$}
		\STATE \textbf{get} distribution $\distfsf_t\gets{w_t}\,/\,{\|w_t\|_1}$; 
		\STATE \textbf{draw} action $a_{i,t}\in\calV_i$ \textbf{from} $\distfsf_t$;
        \STATE \textbf{broadcast} $a_{i,t}$ potentially via multi-hop communication;
		\STATE \textbf{receive} neighbors' actions $\{a_{j, s}\}_{j\in\calN_{i}}$ for $\{s:s+d_i=t\}$;
		\STATE $r_{a_{i,s}, s}\gets f_s(\,a_{i,s}\,|\, \{a_{j, s}\}_{j\in\calN_{i}}\,)$ and \\
  \textbf{normalize $r_{a_{i,s}, s}$ to} $[0,1]$;
        \STATE $\hat{r}_{a,s} \gets 1 - \frac{{\bf 1}(a_{i,s}\,=\,a)}{p_{a,s}}\left(1\,-\,r_{a_{i,s}, s}\right)$, $\forall a\in\calV_i$;
            \STATE $w_{a,t+1}\gets w_{a,t}\exp{(\eta_i \,\hat{r}_{a,s})}, \forall a\in\calV_i$;	
		\ENDFOR
	\end{algorithmic}\label{alg:main}
\end{algorithm}

We present the Distributed Online Greedy algorithm (\alg) for \Cref{pr:online}. 
Particularly, \Cref{pr:online} takes the form of adversarial bandit problems with delayed feedback. Therefore, in the following, we first present the problem formulation of adversarial bandit with delayed feedback (\Cref{subsec:prelim}) and then the main algorithm (\Cref{subsec:action}).

\subsection{Adversarial Bandit with Delayed Feedback}\label{subsec:prelim}

The adversarial bandit with delayed feedback problem involves an agent selecting a sequence of actions to maximize the total reward over a given number of time steps~\cite{thune2019nonstochastic}.  The challenges are: (i) at each time step $t$, no action's reward is known to the agent a priori, and (ii) after an action is selected, only the selected action's reward will become known with a time delay $d_t$, which is assumed to be known a priori. We present the problem in the following using the notation:
\begin{itemize}[leftmargin=*]
    \item $\calV$ denotes the available action set;
    \item $v_{t}\in\calV$ denotes the agent's selected action at $t$;
    \item $r_{v_t,\,t}\in[0,1]$ denotes the reward of selecting $v_{t}$ at $t$;
    \item $d_t$ is the delay for the reward of selecting action $v_t$ at $t$ to be received. That is, the value of $r_{v_t,\,t}$ will be known by the agent at $t+d_t$.
\end{itemize}

\begin{problem}[Adversarial Bandit with Delayed Feedback~\cite{thune2019nonstochastic}]\label{pr:bandit}
Consider a horizon of $T$ time steps. At each time step $t\in[T]$, the agent needs to select an action $v_t\in\calV$ such that the regret
\begin{equation}\label{eq:bandit}
    \operatorname{Regret}_T \triangleq\max_{v\myin\mathcal{V}} \; \sum_{t=1}^T\;r_{v,\,t} \;- \;\sum_{t=1}^T\;r_{v_t,\,t},
\end{equation}
is minimized, where no actions' rewards are known a priori, and only the selected action's reward $r_{v_t,\,t}\in[0,1]$ will become known at $t+d_t$.
\end{problem}

The goal of solving \Cref{pr:bandit} is to achieve a sublinear $\operatorname{Regret}_T$, \ie $\operatorname{Regret}_T/T\rightarrow 0$ for $T\rightarrow \infty$, since this implies that the agent asymptotically chooses optimal actions even though the rewards are unknown a priori~\cite{thune2019nonstochastic}.

\subsection{\alg Algorithm}\label{subsec:action}
We introduce the Distributed Online Greedy (\alg) algorithm (\Cref{alg:main}). \alg enables agents to solve \Cref{pr:online} by simultaneously solving their own instance of \Cref{pr:bandit}. 
To describe the algorithm, we use the notation:
\begin{itemize}[leftmargin=*]
    \item $\calA_t\triangleq \{a_{i,t}\}_{i\in\calN}$ is the set of all agents' actions at $t$;
    \item $\solopt\in\arg\max_{a_{i}\in\mathcal{V}_{i},\, \forall\, i\in\calN} \sum_{t=1}^T f_t(\{a_{i}\}_{i\in \calN})$ is the optimal actions for agents $\calN$ that solve \Cref{pr:online}.
\end{itemize}

Intuitively, our goal is for each agent $i$ at each time step $t$ to efficiently select an action $a_{i,t}$ that maximizes the marginal gain $f_t(\,a\,|\,\{a_{j,t}\}_{j\in\calN_{i}}\,)$. That is, \alg aims to efficiently minimize the following quantification:

\begin{definition}[Static Regret for Each Agent $i$]\label{def:action-regret}
    Given that agent $i$ has multi-hop coordination neighborhood $\calN_{i}$. At each time step $t$, suppose agent $i$ selects an action $a_{i,t}$. Then, the static regret of $\{a_{i,t}\}_{t\in [T]}$ is defined as
    \begin{align}\label{eq:action}
        &\operatorname{Reg}_T\mleft(\{a_{i,t}\}_{t\in [T]}\mright) \triangleq \\\nonumber
        &\max_{a\in\calV_i} \sum_{t=1}^{T} f_t(\,a\,|\,\{a_{j,t}\}_{j\in\calN_{i}}\,) - \sum_{t=1}^{T} f_t(\,a_{i,t}\,|\,\{a_{j,t}\}_{j\in\calN_{i}}\,).
    \end{align}
\end{definition}

Ideally, the agents select actions \textit{simultaneously}, unlike offline algorithms such as \sg~\cite{fisher1978analysis}.
But if the agents aim to select actions simultaneously, $\{a_{j,t}\}_{j\in\calN_{i}}$ will become known only after agent $i$ selects $a_{i,t}$ and communicates with $\calN_{i}$. Therefore, computing the marginal gain is possible only in hindsight, after all agents' decisions have been finalized for time step $t$. Moreover, after $a_{i,t}$ is selected, the feedback $\{a_{j,t}\}_{j\in\calN_{i}}$ cannot be transmitted to agent $i$ until after a delay $d_i$, due to potentially multi-hop communication. Thus, \Cref{pr:online} aligns with the framework of \Cref{pr:bandit} at the single-agent level, where the reward of selecting $a_{i,t}\in\calV_i$ at time $t$, \ie $r_{a_{i,t},t}\triangleq f_t(\,a_{i,t}\,|\,\{a_{j,t}\}_{j\in\calN_{i}}\,)$, will not be known by agent $i$ until time $t+d_i$. 

\alg starts by initializing a learning rate $\eta_i$ and a weight vector $w_t$ for all available actions $a\in\calV_i$ (\Cref{alg:main}'s lines 1--2). Then, at each $t\in[T]$, it sequentially executes the following steps:
\begin{itemize}[leftmargin=*]
    \item Compute probability distribution $p_t$ using $w_{t}$ (lines 3--4);
    \item Select action $a_{i,t}\in\calV_i$ by sampling from $p_t$ (line 5);
    \item Send $a_{i,t}$ to out-neighbors and relay in-neighbors' actions if possible (line 6);
    \item Receive in-neighbors' past actions $\{a_{j,s}\}_{j\in\calN_{i}}$ for $\{s:s+d_i=t\}$, where $d_i$ is the time for all $\{a_{j,s}\}_{j\in\calN_{i}}$ to reach $i$, \ie the communication delay (line 7);
    \item Compute reward $r_{a_{i,s},s}$, estimate reward $\hat{r}_{a,s}$ of each $a\in\calV$, and update weight $w_{a,t+1}$ of each $a\in\calV_i$ (lines 8--11).\footnote{{The coordination algorithms in~\cite{du2022jacobi,rezazadeh2023distributed,robey2021optimal} instruct the agents to select actions simultaneously at each time step as {\fontsize{7}{7}\selectfont\sffamily DOG}, but they lift the coordination problem to the continuous domain and require each agent to know/estimate the gradient of the multilinear extension of $f_t$, which leads to a decision time at least one order higher than {\fontsize{7}{7}\selectfont\sffamily DOG}~\cite{xu2025communication}.}}
\end{itemize}

\section{Approximation Guarantees}\label{sec:bound}

We present the suboptimality bound of \alg. The bound compares \alg's solution to the optimal solution of \Cref{pr:online}. Leveraging the concept of \scenario{coin} (\Cref{def:coin}) that captures the suboptimality cost of decentralization, the bound covers the spectrum of \alg's approximation performance from when the network is fully centralized (all agents communicating with \textit{all}) to fully decentralized (all agents communicating with \textit{none}). 

\begin{definition}[Centralization of Information~\cite{xu2025communication}]
\label{def:coin}
For each time step $t\in [T]$, consider a function $f_t:2^{\calV_\calN}\mapsto$ $\mathbb{R}$ and a communication network $\{\calN_i\}_{i\myin\calN}$ where each agent $i\in \calN$ has selected an action $a_{i,t}$. Then, at time $t$, agent $i$'s \emph{centralization of information} is defined as
\begin{equation}\label{eq:ourcurv}
  \ourcurv_{f_t,i} (\calN_{i})\triangleq f_t(a_{i,t}) - f_t(a_{i,t}\,|\,\{a_{j,t}\}_{j\myin\calN_{i}^c}).
\end{equation}
\end{definition}

$\ourcurv_{f_t,i}$ measures how much $a_{i,t}$ can overlap with the actions of agent $i$'s non-neighbors. In the best scenario, where $a_{i,t}$ does \underline{not} overlap with other actions at all, \ie $f_t(a_{i,t}\,|\,\{a_{j,t}\}_{j\myin\calN_{i}^c})=f_t(a_{i,t})$, then $\ourcurv_{f_t,i}=0$.  In the worst case instead where $a_{i,t}$ is fully redundant, \ie $f_t(a_{i,t}\,|\,\{a_{j,t}\}_{j\myin\calN_{i}^c})=0$, then $\ourcurv_{f_t,i}= f_t(a_{i,t})$.

We also need the following definition to present the approximation performance of \alg.
\begin{definition}[Curvature~\cite{conforti1984submodular}]\label{def:curvature}
        The curvature of a normalized submodular function $f\colon 2^{\calV}\mapsto \mathbb{R}$ is defined as
        \begin{equation}
            \kappa_f\triangleq 1-\min_{\elem\in\calV}{[f(\calV)-f(\calV\setminus\{\elem\})]}/{f(\elem)}.
        \end{equation}
\end{definition}$\kappa_f$ measures how far~$f$ is from modularity. When $\kappa_f = 0$, we have $f(\calV) - f(\calV \setminus \{v\}) = f(v)$ for all $v \in \calV$, \ie the marginal contribution of each element is independent of the presence of other elements, and thus $f$ is modular. In contrast, $\kappa_f = 1$ in the extreme case where there exists some $v \in \calV$ such that $f(\calV) = f(\calV \setminus \{v\})$, \ie $v$~has no contribution in the presence of $\calV\setminus\{v\}$.

\begin{theorem}[Approximation Performance]\label{th:main} 
Over $t\in [T]$, given the communication network $\{\calN_{i}\}_{i\in\calN}$, \alg instructs each agent $i\in\calN$ to select actions $\{a_{i,t}\}_{t\in[T]}$ that guarantee
\begin{align}
    &\mathbb{E}\mleft[\sum_{t=1}^T f_t(\calA_t)\mright] \geq \frac{1}{1+\curv_f}\sum_{t=1}^T f_t(\solopt) \label{eq:thm-4}\\\nonumber
    &- \frac{\kappa_f}{1+\curv_f} \sum_{i\in\calN} \mathbb{E} \mleft[\sum_{t=1}^T \ourcurv_{f_t,i}(\calN_{i})\mright] - \underbrace{\Tilde{O}\mleft(|\calN|\sqrt{{\overline{|\calV|\,+\,d}}\,{T}}\mright)}_{\psi(T)},
\end{align}where $\kappa_{f}\triangleq\max_{t\in [T]}\kappa_{f_t}$, $\overline{|\calV|\,+\,d} \triangleq\max_{i\in\calN}\, (|\calV_i|\, +\, d_i)$, $\bar{d}\triangleq\max_{i\in\calN} d_i$, the expectation is due to \alg's internal randomness, and $\tilde{O}(\cdot)$ hides $\log$~terms.

\end{theorem}
In particular, when the network is fully centralized,  \ie $\calN_{i}\equiv\calN\setminus\{i\}$,
    {\begin{equation}\label{eq:thm-1}
        \mathbb{E}\mleft[\sum_{t=1}^T f_t(\calA_t)\mright] \geq \frac{1}{1+\curv_f}\sum_{t=1}^T f_t(\solopt) - \underbrace{\Tilde{O}\mleft(|\calN|\sqrt{\overline{|\calV|\,+\,d}\,{T}}\mright)}_{{\phi(T)}}.
    \end{equation}}
When the network is fully decentralized, \ie $\calN_{i}\equiv\emptyset$, 
    {\begin{equation}\label{eq:thm-2}
        \mathbb{E}\mleft[\sum_{t=1}^T f_t(\calA_t)\mright] \geq (1-\curv_f)\sum_{t=1}^T f_t(\solopt) -\underbrace{\Tilde{O}\mleft(|\calN|\sqrt{\overline{|\calV|\,+\,d}\,{T}}\mright)}_{{\chi(T)}}.
    \end{equation}}

In all, as $T\to\infty$, \alg enables asymptotically near-optimal action coordination. Particularly, 
\Cref{th:main} quantifies both the convergence speed of \alg and the suboptimality of \alg due to decentralization:
\begin{itemize}[leftmargin=*]
    \item \textit{Convergence time}: $\frac{\psi(T)}{T}, \frac{\phi(T)}{T}, \frac{\chi(T)}{T}$ in \cref{eq:thm-1,eq:thm-2,eq:thm-4} capture the time needed for action selection to converge to near-optimality and its impact to the suboptimality bound. They vanish as $T\to\infty$, having no impact on the suboptimality bound anymore. %
   
    \item \textit{Decentralization}: After $\frac{\psi(T)}{T}$ vanishes as $T\to\infty$, the bound in \cref{eq:thm-4} depends on $\scenario{coin}_{f_t,i}$ capturing the suboptimality due to decentralization: the larger is $\calN_i$ for each $i\in\calN$, the smaller is $\scenario{coin}_{f_t,i}$, and the higher is \alg's approximation performance.  
    That is, \cref{eq:thm-1,eq:thm-2,eq:thm-4} imply \alg's suboptimality will improve if the agents have larger multi-hop coordination neighborhoods. %
    Importantly, the $1/(1+\curv_f)$ suboptimality bound with a fully connected network recovers the bound in~\cite{conforti1984submodular} and is near-optimal as the best possible bound for~\eqref{eq:problem} is $1-\kappa_f/e$~\cite{sviridenko2017optimal}.\footnote{{The bounds $1/(1+\curv_f)$ and $1-\kappa_f/e$ become $1/2$ and $1-1/e$ when, in the worst case, $\kappa_f=1$.}}%
\end{itemize}

\section{Runtime Analysis}\label{sec:resource-guarantees}
We present the runtime of \alg by analyzing its computation and communication complexity (accounting for message length). 
We use the notation and observations:
\begin{itemize}[leftmargin=*]
    \item $\tau_f$ is the time required for one evaluation of $f$;
    \item $\tau_c$ is the time for {transmitting the information about one action} from an agent directly to another agent;
    \item $\epsilon$ is the convergence error after $T$ iterations: $T\geq |\calN|^2/\epsilon$ is required for $\frac{\psi(T)}{T}, \frac{\phi(T)}{T}, \frac{\chi(T)}{T} \leq \epsilon$ per \cref{eq:thm-1,eq:thm-2,eq:thm-4}.
\end{itemize}

\begin{proposition}[Computational Complexity]\label{prop:computation}
At each $t\in [T]$, \alg requires each agent $i$ to execute $2$ function evaluations of $f_t$ and $O(|\calV_i|)$ additions and multiplications.
\end{proposition}
\begin{proof}
    At each $t\in[T]$, \alg requires $2$ function evaluations To compute the marginal gain (\Cref{alg:main}'s line 8), along with $O(|\calV_i|)$ additions and multiplications  (\Cref{alg:main}'s lines 4 and 9--10), and thus \Cref{prop:computation} holds.
\end{proof}

\begin{proposition}[Communication Complexity]\label{prop:communication}
At each $t\myin [T]$, \alg requires $O\mleft(\tau_c\, \bar{d}\mright)$ communication time such that each agent can transmit enough actions throughout its coordination neighborhood without information congestion.
\end{proposition}
\begin{proof}
\Cref{prop:communication} holds since, at each $t\in[T]$, if the communication volume is less than $O(d_i)$ actions for agent $i$, then there will be newly selected actions congesting in the network, leading to an increasing amount of feedback delays (\Cref{alg:main}'s lines 6--7). 
\end{proof}

\begin{theorem}[Convergence Time]\label{th:speed}
\alg achieves $\epsilon$-convergence to near-optimal actions in $\Tilde{O}\mleft[\mleft(\tau_f+\tau_c\,\bar{d}\mright)\,\overline{|\calV|\,+\,d}\,|\calN|^2/\epsilon\mright]$.
\end{theorem}
\begin{proof}
\Cref{th:speed} holds by combining  \Cref{prop:computation,prop:communication}, along with the definition of $\epsilon$ above, upon ignoring the time needed for additions and multiplications.
\end{proof}

\begin{remark}[Trade-off Between Coordination Performance and Convergence Time]\label{rem:tradeoff}
We observe this trade-off from \Cref{th:main,th:speed}, determined by the magnitude of communication delays. Incorporating delayed information from larger multi-hop neighborhoods improves approximation performance yet increasing also the convergence time, while restricting communication to one-hop neighbors only accelerates convergence at the expense of lower coordination performance. 
In the fully centralized case, \ie the delay takes the largest value $|\calN|-1$, \alg will recover \banalg's approximation bound of $1/(1+\kappa_f)$, while converging faster than \banalg by $O(|\calN|)$. In the one-hop coordination case, \alg will recover \actionsel's approximation bound with the same convergence time. 
\end{remark}

\section{Conclusion} \label{sec:con}
We presented the Distributed Online Greedy (\alg) algorithm for multi-agent submodular maximization under communication delays. Leveraging tools from adversarial bandit with delayed feedback, \alg enables agents to make simultaneous online decisions while incorporating delayed feedback information from multi-hop neighbors, maximizing each agent's coordination neighborhood. We provided the approximation guarantees that capture the suboptimality cost of network decentralization and showed that \alg enables agents to select asymptotically near-optimal actions. %
The analyses of approximation bounds and convergence time that revealed a trade-off between coordination performance and convergence time: as the communication delay decreases, \alg covers the spectrum between fully centralized coordination and one-hop coordination. 

\myParagraph{Future work} We will extend this work to enable the agents to actively address the trade-off of coordination performance and convergence rate, by tuning their admissible communication delays.

\bibliographystyle{IEEEtran}
\bibliography{references}

\appendices

\section{Proof of~\Cref{th:main}}\label{app:main}

We prove the main result: %
{\allowdisplaybreaks\begin{align}
    &\sum_{t=1}^{T} f_t(\calA^\opt)\nonumber\\
    &=\sum_{t=1}^{T} f_t(\calA^\opt\cup\calA_t) - \sum_{t=1}^{T} \sum_{i\in\calN} f_t(a_{i,t}\,|\,\calA^\opt\cup\{a_{j,t}\}_{j\in[i-1]}) \label{aux22:1}\\
    &\leq\sum_{t=1}^{T} f_t(\calA_t) + \sum_{t=1}^{T} \sum_{i\in\calN} f_t(a_{i}^\opt\,|\,\calA_t) \nonumber\\
    &\quad - (1-\kappa_{f}) \sum_{t=1}^{T} \sum_{i\in\calN} f_t(a_{i,t}\,|\,\{a_{j,t}\}_{j\in\calN_{i}}) \label{aux22:2}\\
    &\leq\sum_{t=1}^{T} f_t(\calA_t) + \kappa_{f} \sum_{t=1}^{T} \sum_{i\in\calN} f_t(a_{i,t}\,|\,\{a_{j,t}\}_{j\in\calN_{i}})\label{aux22:3}\\ \nonumber
    &\quad + \sum_{i\in\calN} \sum_{t=1}^{T} \mleft[f_t(a_{i}^\opt\,|\,\{a_{j,t}\}_{j\in\calN_{i}}) - f_t(a_{i,t}\,|\,\{a_{j,t}\}_{j\in\calN_{i}})\mright] \\  
    &\leq\sum_{t=1}^{T} f_t(\calA_t) + \sum_{i\in\calN} \Reg(\{a_{i,t}\}_{t\in [T]}) \nonumber\\
    &\quad + \kappa_{f} \sum_{t=1}^{T} \sum_{i\in\calN} f_t(a_{i,t}\,|\,\{a_{j,t}\}_{j\in\calN_{i}}) \label{aux22:4}\\
    &=(1+\kappa_f) \sum_{t=1}^{T} f_t(\calA_t) + \sum_{i\in\calN} \Reg(\{a_{i,t}\}_{t\in [T]}) \label{aux22:5}\\\nonumber
    &\quad+ \kappa_{f} \sum_{t=1}^{T} \sum_{i\in\calN} \mleft[f_t(a_{i,t}\,|\,\{a_{j,t}\}_{j\in\calN_{i}}) - f_t(a_{i,t}\,|\,\{a_{j,t}\}_{j\in [i-1]})\mright] \\
    &\leq(1+\kappa_f) \sum_{t=1}^{T} f_t(\calA_t) + \sum_{i\in\calN} \Reg(\{a_{i,t}\}_{t\in [T]}) \nonumber\\
    &\quad+ \kappa_{f} \sum_{t=1}^{T} \sum_{i\in\calN} \mleft[f_t(a_{i,t}) - f_t(a_{i,t}\,|\,\{a_{j,t}\}_{j\in [i-1]\setminus\calN_{i}})\mright] \label{aux22:6}\\
    &\leq (1+\kappa_f) \sum_{t=1}^{T} f_t(\calA_t) + \sum_{i\in\calN} \Reg(\{a_{i,t}\}_{t\in [T]}) \nonumber\\
    &\quad+ \kappa_{f} \sum_{t=1}^{T} \sum_{i\in\calN} \underbrace{\mleft[f_t(a_{i,t}) - f_t(a_{i,t}\,|\,\{a_{j,t}\}_{j\in \calN_{i}^c})\mright]}_{\ourcurv_{f_t,i}(\calN_{i})}, \label{aux22:7}
\end{align}}where \cref{aux22:1} holds by telescoping the sum, \cref{aux22:2} holds since $f$ is submodular and since $1-\kappa_{f} \leq \frac{f_t(a_{i,t}\,|\,\calA^\opt\,\cup\,\{a_{j,t}\}_{j\in\calN\setminus\{i\}})}{f_t(a_{i,t})} \leq \frac{f_t(a_{i,t}\,|\,\calA^\opt\,\cup\,\{a_{j,t}\}_{j\in[i-1]})}{f_t(a_{i,t}\,|\,\{a_{j,t}\}_{j\in\calN_{i}})}$ per \Cref{def:curvature}, \cref{aux22:3} holds from submodularity, \cref{aux22:4} holds from \Cref{def:action-regret}, \cref{aux22:6} holds since $f_t$ is 2nd-order submodular, and \cref{aux22:7} holds from \Cref{def:coin}. %

Taking expectations on both sides of \cref{aux22:7} and leveraging~\cite[Theorem 1]{thune2019nonstochastic}, we prove \cref{eq:thm-4} by the following,
{\begin{align}
    &\sum_{t=1}^T f_t(\solopt) \leq (1+\kappa_f) \,\mathbb{E}\mleft[\sum_{t=1}^T f_t(\calA_t)\mright]\nonumber \\
    & + \kappa_f \sum_{i\in\calN} \mathbb{E}\mleft[\sum_{t=1}^T\ourcurv_{f_t,i}(\calN_{i})\mright] + \Tilde{O}\mleft(|\calN|\sqrt{\overline{|\calV|\,+\,d}\,{T}}\mright).\label{aux22:10}
\end{align}}

In the fully centralized scenario, we have $\calN_{i}=\calN\setminus\{i\}$. Thus, $\ourcurv_{f_t,i}(\calN_{i})= 0$, and thus \cref{eq:thm-1} is proved.
Finally, in the fully decentralized case where $\calN_{i}=\emptyset$, per \cref{aux22:4}, 
{\allowdisplaybreaks
\begin{align}
    \mathbb{E}\mleft[\sum_{t=1}^T f_t(\calA_t)\mright]&\geq \sum_{t=1}^T f_t(\solopt) - \kappa_f \sum_{i\in\calN} \mathbb{E}\mleft[\sum_{t=1}^T f_t(a_{i,t})\mright] \nonumber\\
    &\quad - \Tilde{O}\mleft(|\calN|\sqrt{\overline{|\calV|\,+\,d}\,{T}}\mright)
    \nonumber\\
    &\hspace{-.6cm}\geq \sum_{t=1}^T f_t(\solopt) - \frac{\kappa_f}{1-\kappa_f} \sum_{i\in\calN} \mathbb{E}\mleft[\sum_{t=1}^T f_t(a_{i,t})\mright] \nonumber\\
    &\quad - \Tilde{O}\mleft(|\calN|\sqrt{\overline{|\calV|\,+\,d}\,{T}}\mright).
\end{align}}and thus \cref{eq:thm-2} is proved. \qed

\end{document}